\documentclass[10pt,twocolumn,letterpaper]{article}

\usepackage[utf8]{inputenc}
\usepackage[colorlinks=true, urlcolor=blue]{hyperref}

\usepackage{amsfonts}
\usepackage{amssymb,amsmath,mathrsfs}
\usepackage{amsthm}
\usepackage{bbm}
\usepackage{mathtools}
\usepackage{datetime}
\usepackage{dsfont}
\usepackage{verbatim}
\usepackage{enumerate}
\usepackage{MnSymbol}
\usepackage[titletoc, title]{appendix}
\usepackage{caption}
\usepackage{float}
\usepackage{arydshln}
\usepackage{fancyhdr}
\usepackage{graphicx}
\usepackage{subfigure}

\usepackage[en-US]{datetime2}
\usepackage[nottoc, notlot, notlof]{tocbibind}

\usepackage[round]{natbib}

\newtheorem{theorem}{Theorem}[section]
\newtheorem*{theorem*}{Theorem}
\newtheorem{proposition}[theorem]{Proposition}
\newtheorem{lemma}[theorem]{Lemma}

\newcommand{\wrt}[1]{\, \mathrm{d} #1}

\newcommand{\norm}[1]{\left\lVert#1\right\rVert}

\DeclareMathOperator{\argmax}{arg\,max}

\title{Bayesian optimization for backpropagation in Monte-Carlo tree search}

\author{ Yueqin Li \\ yueqin\_li@mymail.sutd.edu.sg \and Nengli Lim \\ nengli\_lim@sutd.edu.sg}

\begin{document}

\maketitle

\begin{abstract}
In large domains, Monte-Carlo tree search (MCTS) is required to estimate the values of the states as efficiently and accurately as possible. However, the standard update rule in backpropagation assumes a stationary distribution for the returns, and particularly in min-max trees, convergence to the true value can be slow because of averaging. We present two methods, Softmax MCTS and Monotone MCTS, which generalize previous attempts to improve upon the backpropagation strategy. We demonstrate that both methods reduce to finding optimal monotone functions, which we do so by performing Bayesian optimization with a Gaussian process (GP) prior. We conduct experiments on computer Go, where the returns are given by a deep value neural network, and show that our proposed framework outperforms previous methods.
\end{abstract}

\vspace{-0.3cm}
\section{Introduction}
\vspace{-0.1cm}
Monte-Carlo tree search (MCTS) \citep{coulom2006efficient, browne2012survey}, or more specifically its most common variant UCT (Upper Confidence Trees; see Section 2) \citep{kocsis}, has seen great successes recently and has propelled, especially in combination with deep neural networks, the performance of computer Go past professional levels \citep{silver2016mastering, silver2017mastering}.
The robust nature of MCTS, versus a traditional approach like depth-first search in alpha-beta pruning, has not only enabled a leap-frog in performance in computer Go, but has also led to its utilization in other games where it is difficult to evaluate states, as well as in other domains \citep{browne2012survey}.

However, MCTS is known to suffer from slow convergence in certain situations \citep{coquelin2007bandit}, in particular when the precise calculation of a narrow tactical sequence is critical for success. For example in boardgames, \citep{ramanujan2010adversarial} defines a level-$k$ search trap for player $p$ after a move $m$ as a state of the game where the opponent of $p$ has a guaranteed $k$-move winning strategy. More relevantly, they show through a series of experiments that MCTS performs poorly even in shallow traps, in contrast to regular minimax search; see also \citep{ramanujan2011behavior, ramanujan2011trade}.

To better understand this phenomenon, we take a closer look at the update rule
\begin{align} \label{origUpdate}
Q_n \leftarrow Q_{n-1} + \frac{R_{n-1} - Q_{n-1}}{n}
\end{align}
which is performed during the backpropagation phase of MCTS. Here, the current estimate of the value of a state is taken to be the simple average of all previous returns accrued upon visiting that state. Proceeding, we discuss various methods which seek to improve backpropagation by challenging the basic assumptions implied by \eqref{origUpdate}:
\begin{enumerate} [(i)]
\item Value estimation by averaging returns: \\
Instead of updating a parent node's value with that of its MAX (MIN) child as in minimax search, backpropagation in MCTS averages all returns to obtain a good signal in noisy environments (this is equivalent to setting the value of the parent node to be the weighted average (by visits) of its children's values).
\item Stationarity: \\
The returns are assumed to follow a stationary distribution.
\end{enumerate}
With regard to the first point, one of the first published works on MCTS \citep{coulom2006efficient} posits that taking the value of the best child leads to an overestimation (cf. \citep{cohen}) of the value of a MAX node, whereas taking the weighted average (by number of visits) of the children's values leads to an underestimation. The paper proposes using an interpolated value, with weights dependent on the current number of visits of the best child:
\begin{align} \label{coulom}
Q_{parent} = \left( \frac{N_{best}}{N_{best} + M} \right) Q_{best} + \left( 1 - \frac{N_{best}}{N_{best} + M}\right) Q_{mean}.
\end{align}
Here, $N_{best}$ and $Q_{best}$ respectively denote the number of visits and backed-up value of the best child, and $M$ is a variable which slowly increases after some fixed threshold to dampen the increasing weight. \par 

Similarly in \citep{khandelwal2016analysis}, a backup strategy MaxMCTS($\lambda$) is proposed where an eligibility parameter $\lambda$ can be adjusted to strike a balance between taking the weighted average of the children's values ($\lambda = 1$) and taking the value of the best child ($\lambda = 0$). In addition, they show that the optimal value for $\lambda$ depends on the context; e.g. in Grid World experiments, it is demonstrated that the more obstacles that are present in the grid, the more $\lambda$ has to be lowered in order to maintain good performance. This corresponds with the findings in \citep{ramanujan2010adversarial, ramanujan2011trade, ramanujan2011behavior}, in that standard MCTS may perform well in environments, for example in the opening stages of Go, where global strategy is more important, but its performance tends to degrade in highly tactical situations; see also \citep{baier}. \par 

Moving on to the second premise, while stationarity may be a viable assumption in multi-armed bandit problems, and although MCTS can be viewed as a sequential multi-armed bandit problem, it is evident that the later simulations explore a larger tree than the earlier simulations. This implies that the sequence of rewards follows a non-stationary distribution, where the returns from later simulations are more informative than the earlier ones, and hence it would be natural to weight them more heavily. \par 

One way of doing this is to simply employ the exponential recency-weighted average update (ERWA) \citep{sutton2018reinforcement} where \eqref{origUpdate} is replaced by
\begin{align} \label{erwa}
Q_n \leftarrow Q_{n-1} + \alpha \left( r(n) - Q_{n-1} \right), \qquad \alpha \in (0, 1];
\end{align} 
see also \citep{hashimoto2011accelerated} where they employ a similar backup strategy. \par

A more sophisticated method called feedback adjustment policy is explored in \citep{xie2009backpropagation}, where here they test four different weight profiles of varying shapes. The following figure provides an illustration. \\

\begin{figure}[H]
\vspace{-0.7cm}
\centering	
\includegraphics[scale=0.55]{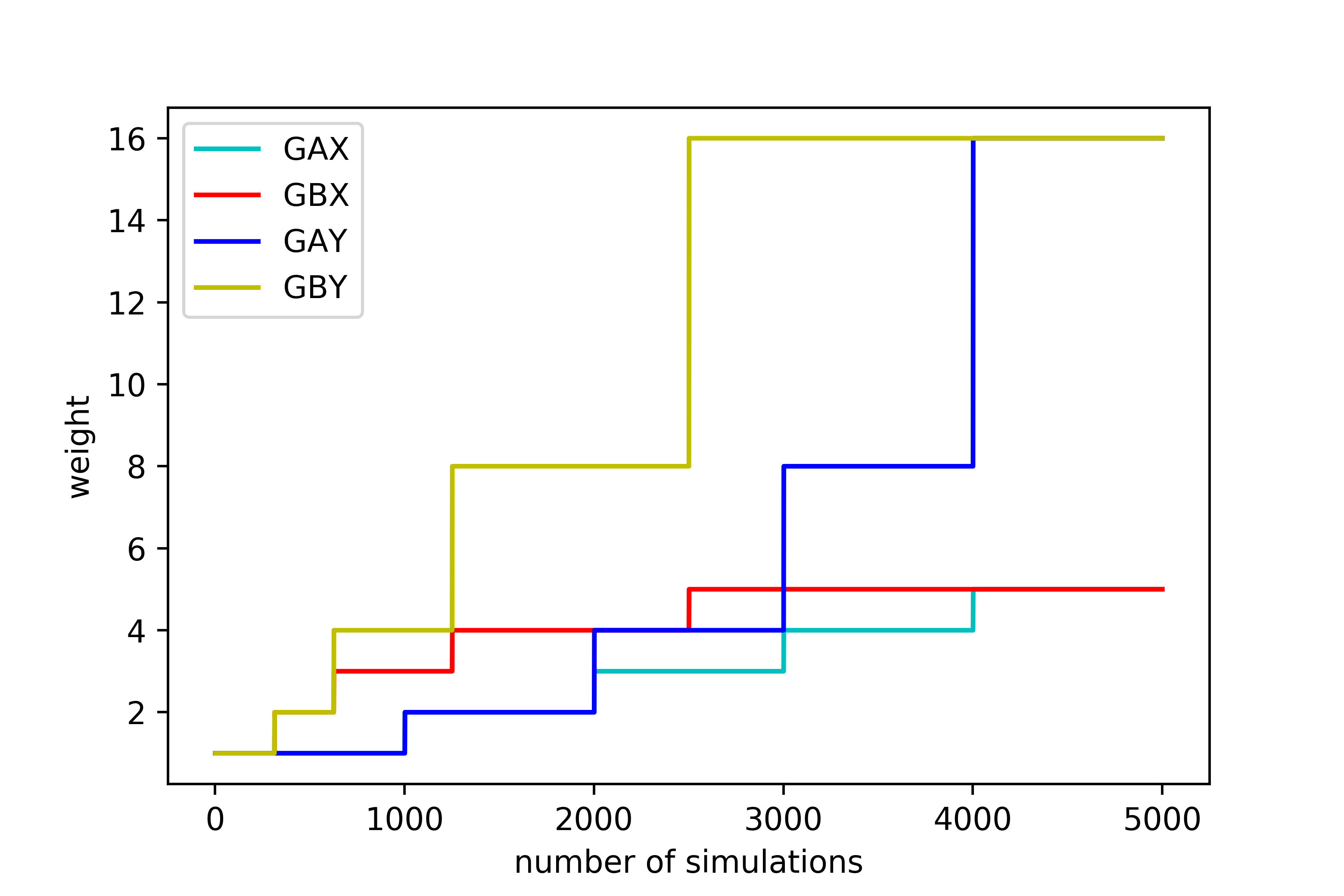}
\caption{Graph depicting the weight increase of four profiles in feedback adjustment policy; $GAX$: linear increase on a uniform partition; $GAY$: exponential increase on a uniform partition; $GBX$: linear increase on a partition with exponentially increasing widths; $GBY$: exponential increase on a partition with exponentially increasing widths.}
\end{figure}

\vspace{-0.2cm}

Experiments on 9x9 Go show that $GBY$ gives the highest winning rate over original MCTS. Just as importantly, they show that in spite of the fact that the functions all monotonically weight the later simulations more heavily, the differences in their particular shapes have a big impact on performance; $GBX$ was found to provide no significant advantage over standard MCTS, in contrast to a 26.6\% boost when using $GBY$ \citep{xie2009backpropagation}.

Despite differences in these various methods, we can summarize the overarching principles they have in common as follows:
\begin{enumerate} [(i)]
\item the best child should be weighted more heavily as the number of simulations increase;
\item later simulations should be weighted more strongly than earlier ones.
\end{enumerate}
Taking this into account, in this paper we propose Monotone MCTS and Softmax MCTS, two backpropagation strategies which aim to generalize and improve upon the previous methods. We first represent the weights as a function of the number of visits of the node in question, and naturally constrain it to be a monotone function. We then propose to use black-box Bayesian optimization to find these optimal monotone functions.  \par

The rest of the paper is structured as follows. In Section 2, we give a brief review of MCTS and Bayesian optimization using a Gaussian process prior. In Section 3, we go into the details of Monontone MCTS and Softmax MCTS. We show the effectiveness of our approach in experiments on 9x9 and 19x19 Go in Section 4. Finally, we conclude in the last section with some direction on future work.

\vspace{-0.3cm}
\section{Preliminaries}
\vspace{-0.1cm}
\subsection{Monte-Carlo tree search}

In comparison to depth-first search in alpha-beta pruning, MCTS uses best-first search to gather information for planning the next action. This is important particularly in computer Go where the branching factor is large, and the tree is best explored asymmetrically to strike a balance between searching deep sequences in tactical situations and searching wide options in factoring in strategic considerations. This also allows it to be an anytime algorithm, in that terminating the search prematurely can still yield acceptable results.

MCTS consists of the following four steps:

\begin{enumerate} [(i)]
\item Selection: Starting from the root node, the search process descends down the tree by successively selecting child nodes according to the \textit{tree policy}. In particular, if Upper Confidence Bound 1 (UCB1)
\begin{align} \label{ucb1}
\underset{a}{\argmax} \quad Q_a  + C \, \sqrt{\frac{\ln N_{parent}}{N_a + 1}}
\end{align}
is used as the tree policy, then this variant of MCTS is called UCT \citep{kocsis}. More recently, PUCT \citep{silver2016mastering, auger2013continuous}
\begin{align} \label{puct}
\underset{a}{\argmax} \quad Q_a  + C \, \frac{\sqrt{N_{parent}}}{N_a + 1}
\end{align}
has been gaining popularity. Here, $Q_a$ and $N_a$ denote the mean and visits respectively of child $a$, and $C$ is a constant that can be tuned to balance exploration vs exploitation.
\item Expansion: When the simulation phase reaches a leaf node, children of the leaf node are added to the tree, and one of them is selected by the tree policy.
\item Simulation: One (or multiple) random playout is performed until a terminal node is reached. More recently, simulation can be augmented or even replaced by a suitable evaluation function such as a neural network.
\item Back-propagation: The result of the playout is computed and \eqref{origUpdate} is used to update each node visited in the selection phase.
\end{enumerate}
Averaging the results in each node is essential in noisy environments and when it is critical not to back up values in a manner such that outliers affect the algorithm adversely. However, it can be slow to converge to the optimal value in min-max trees, particularly in nodes where the siblings of an optimal child node are all lower in value \citep{fuThesis}.

\subsection{Bayesian optimization with a Gaussian process prior}
Given an index set $T$, $\left\{ X_t; t\in T \right\}$ is a Gaussian process if for any finite set of indices $\left\{t_1,...,t_n \right\}$ of $T$, $\left(X_{t_1},...,X_{t_n} \right)$ is a multivariate normal random variable. By specifying a mean function $\mu : \mathbb{R}^d \rightarrow \mathbb{R}$ and a symmetric, positive semi-definite kernel function $k:\mathbb{R}^d \times \mathbb{R}^d \rightarrow \mathbb{R}$, one can uniquely define a Gaussian process by setting
\begin{align*}
\left(X_{t_1},...,X_{t_n} \right) \sim \mathcal{N}([\mu(x_1), \ldots, \mu(x_n)]^T, K)
\end{align*}
for any finite subset $\left\{t_1,...,t_n \right\}$ of $T$. Here, the covariance function $K$ refers to
\begin{align*}
K = \left[
\begin{matrix}
k(x_1,x_1) & \cdots & k(x_1,x_n)\\
\vdots & \ddots & \vdots\\
k(x_n,x_1) & \cdots & k(x_n,x_n)\\
\end{matrix}
\right].
\end{align*}
In addition, we assume that the model $f$ is perturbed with noise,
\begin{align*}
y = f(x) + \epsilon,
\end{align*}
where $\epsilon \sim \mathcal{N}(0, \tau^2)$, and is assumed to be independent between samples. \par 

In many machine learning problems, the objective function $f$ to optimize is a black-box function which does not have an analytic expression, or may have one that is too costly to compute. Hence, a Gaussian processes are used as surrogate models to approximate the true function as they yield closed-form solutions. For example, if we stipulate that $f(\cdot) \sim \mathcal{N}(0, k)$, then given a history of input-observation pairs  $\left( x^{(1)}, t^{(1)} \right), \ldots, \left( x^{(n)}, t^{(n)} \right)$, and a new input point $x^{(n+1)}$, we can predict $y^{(n+1)}$ by computing the posterior distribution
\begin{align*}
p\left(y^{(n+1)} \, \Big| \, y^{(1)}=t^{(1)}, \ldots, y^{(n)}=t^{(n)} \right),
\end{align*}
which is Gaussian with mean $\mu$ and variance $\sigma^2$ given by the formulas
\begin{align} \label{meanCov}
\mu = r^T(K^n)^{-1}t_n,\\
\sigma^2 = c - r^T(K^n)^{-1}r.
\end{align}
Here, we denote
\begin{align*}
&r = \left[ k\left( x^{(1)}, x^{(n+1)} \right),\ldots, k\left( x^{(n)},x^{(n+1)} \right) \right]^T, \\
&c = k \left( x^{(n+1)},x^{(n+1)} \right) + \tau^2, \\
&t_n = \left[t^{(1)}, \ldots, t^{(n)}\right]^T,
\end{align*}
and $K^n$ is the covariance matrix corresponding the first $n$ inputs. For more information on Gaussian processes for machine learning, we refer the reader to \citep{williams2006gaussian}. \par 

Another reason for Bayesian optimization becomes apparent when finding the optimal value of $f$ is costly, for example in high-dimensional problems where performing a grid search to find the optimal value is prohibitive, eg. hyper-parameter tuning in deep learning models. \par 

In such cases, an \textit{acquisition function} is selected to guide sampling to areas where one will have an increased probability of finding the optimum. Two common examples are Expected Improvement(EI)
\begin{align*}
A(x,f^*) &= \mathbb{E}\left[max\left\{f_x - f^*,0\right\}\right] \\
&= \sigma_x\left[\gamma_x\Phi(\gamma_x)+\phi(\gamma_x)\right], \\
&\hspace{-3em} \gamma_x = \frac{\mu_x - f^*}{\sigma_X}, \quad f_x \sim \mathcal{N}(\mu_x, \sigma_x^2),
\end{align*}
where $f^*$ denotes the maximum value of $f$ found so far, and Upper Confidence Bound (UCB)
\begin{align*}
A(x) = \mu_x + \kappa \sigma_x.
\end{align*}
In both examples, $\mu_x$ and $\sigma_x$ are obtained from \eqref{meanCov}, and there is a trade-off between exploration and exploitation in the selection of the next point. \par 

For greater efficiency, we use Spearmint \citep{snoek2012practical}, which allows the optimization procedure to be run in parallel on multiple cores. Spearmint adopts the Mat\'{e}rn $\frac{5}{2}$ kernel 
\begin{align*}
&K_{M52} (x, y)  \\
&\quad= c \left(1 + \sqrt{5 \norm{x - y}^2} + \frac{5}{3} \norm{x - y}^2 \right) e^{-\sqrt{5 \norm{x - y}^2}}
\end{align*}
for the Gaussian process prior, and chooses the next point based on the expected acquisition function under all possible outcomes of the pending evaluations. It was shown to be effective for many algorithms including latent Dirichlet allocation and hyper-parameter tuning in convolutional neural networks \citep{snoek2012practical}. \par 

In the context of computer Go, Bayesian optimization with a Gaussian process prior has also previously been used in \citep{alphaGoOpt}. With regard to MCTS, they perform optimization over the UCT exploration constant and the mixing ratio between fast rollouts and neural network evaluations.

\vspace{-0.3cm}
\section{Methods}
\vspace{-0.1cm}
To find optimal backpropagation strategies, we first parameterize a family of smooth monotone functions, then perform Bayesian optimization with a Gaussian process prior as reviewed in the previous section. To describe this family of functions, we invoke the following lemma.

\begin{lemma}
$w: [a, b] \rightarrow \mathbb{R}$ is continuously differentiable and strictly monotonic if and only if there exists a continuous function $p:[a, b] \rightarrow \mathbb{R}$ such that
\begin{align} \label{integral}
w(t) = w(a) + \int_a^t e^{p(s)} \wrt{s}.
\end{align}
\end{lemma}

\begin{proof}
By the mean-value theorem, $w'(t) > 0$ for all $t \in (a, b)$. Thus we have
\begin{align*}
w(t) 
&= w(a) + \int_a^t w'(s) \wrt{s} \\
&= w(a) + \int_a^t e^{\log w'(s)} \wrt{s}.
\end{align*}
\end{proof}
Finding the optimal monotone function, even when restricted to continuously differentiable ones, is a functional Bayesian optimization \citep{vien2018bayesian} problem as the optimization is taking place over an infinite dimensional Hilbert space of functions. However, for practical reasons, we instead restrict the class of functions we are optimizing over to be
\begin{align*}
\mathcal{W} = \left\{ w \:\: \bigg| \:\: w(t) = w(0) + \int_0^t e^{p(s)} \wrt{s}, \, p \in \mathcal{P} \right\},
\end{align*}
where
\begin{align*}
&\mathcal{P} = \left\{ p \:\: \bigg| \:\: p(s) = \sum_{i=0}^{m-1} \mathds{1}_{\left[ i \Delta, (i+1) \Delta \right)} (s)
\tilde{p}_i(s) \right\}, \\
&\tilde{p}_i(s) 
:= p \left( i \Delta \right)  + \left( \frac{s}{\Delta}  - i \right) \left( p \left( (i+1) \Delta \right) - p \left( i \Delta \right) \right),
\end{align*}
is the $m$-dimensional space of functions obtained from linearly interpolating between the $m$ points $\left( p(0), p\left(\Delta\right), \ldots, p(N) \right)$ which are uniformly separated by an interval $\Delta = \frac{N}{m-1}$, and $N$ denotes the number of simulations. 

\subsection{Monotone MCTS}
The first backpropagation strategy we propose is Monotone MCTS. We first run the optimization procedure, with respect to win-rate, over $m$ parameters $\left\{ p(0), p\left(\Delta\right), \ldots, p(N) \right\}$. Each set of parameters yields a continuous function $p \in \mathcal{P}$ by interpolation and a monotone weight function $w$ using \eqref{integral} with $w(0)$ set to 1. Upon choosing the optimal set of parameters, the update rule is then modified to be
\begin{align} \label{monotoneMCTS}
Q_n \leftarrow \frac{w(0)}{S(n)} r(0) + \frac{w(1)}{S(n)} r(1) + \cdots + \frac{w(n)}{S(n)} r(n),
\end{align}
where we denote $S(n) := \sum_{t=0}^n w(t)$. \par 

Despite being a subset of all possible monotone functions, we consider $\mathcal{W}$ to be sufficiently rich as it contains all increasing linear functions (starting at 1)
\begin{align*}
(c, c, \ldots, c) \implies w(t) = 1 + e^c t,
\end{align*}
all exponential functions
\begin{align*}
&(\log (ra), \log(ra) + r \Delta, \log(ra) + 2 r \Delta, \ldots) \\
&\qquad \implies  w(t) = (1 - a) + a e^{rt},
\end{align*}
as well as their linear combinations and other monotone functions such as those in \citep{xie2009backpropagation}. As an example, the following simple  Proposition shows how to convert between ERWA with parameter $\alpha$ and our formulation.
\begin{proposition}
ERWA with parameter $\alpha < 1$ can be obtained by setting $p(s) = (\log \lambda) s + \log \left( \alpha \log \lambda \right)$, where $\lambda := \frac{1}{1 - \alpha}$.
\end{proposition}
\begin{proof}
We can expand \eqref{erwa} to obtain
\begin{align*} 
Q_n 
&= (1 - \alpha)^n \, r(0) + \alpha (1 - \alpha)^{n-1} \, r(1) + \cdots + \alpha \, r(n).
\end{align*}
To obtain the weights $w$, we now simply compare coefficients with \eqref{monotoneMCTS} to derive
\begin{align*}
w(t) 
&= \alpha (1 - \alpha)^{-t} \\
&= \alpha + \int_0^t e^{(\log \lambda) s + \log \left( \alpha \log \lambda \right)} \wrt{s}.
\end{align*}
\end{proof}

\vspace{-1cm}
\subsection{Softmax MCTS}

For our second backpropagation strategy, we draw inspiration from the softmax distribution
\begin{align*}
p(x_1, \ldots, x_d)_i = \frac{e^{w x_i}}{\sum_{j = 1}^d e^{w x_j}},
\end{align*}
which converges as $w \rightarrow \infty$ to $e_k = (0, \ldots, 1, \ldots 0)$, with 1 in the $k^{th}$ position, when $x_k$ is the maximum of $\{ x_i \}$. We develop a new robust method, Softmax MCTS, for interpolating between the theoretical minimax value of the node and the original averaged value in standard MCTS as follows. \par 
Let $Q_j$ and $N_j$ respectively denote the mean and number of visits of the $j^{th}$ child. In Softmax MCTS, we define the backpropagation update after every simulation for every parent node as
\begin{align*}
Q_{parent} \leftarrow \frac{\sum_{j=1}^d \alpha_j Q_j}{\sum_{j=1}^d \alpha_j},
\end{align*}
where
\begin{align*}
\alpha_j = N_j \, e^{Q_j w \left( N_{parent} \right)}.
\end{align*}
Here, $w$ is a monotonically increasing function of the number of visits of the parent node, which will be optimized in the same manner as given in the previous subsection, with the difference that now $w(0)$ is set to $0$. In early stages when $w$ is close to 0, $w_j$ is approximately $N_j$, which means that
\begin{align*}
Q_{parent} \approx \frac{\sum_{j=1}^d N_j Q_j}{\sum_{j=1}^d N_j}.
\end{align*}
This is equivalent to the weighted-average update rule of standard MCTS. As $w$ increases with the number of visits, the weights will gradually favour the child with the maximum mean (minimum if the parent is a MIN node). \par 
We believe that this is more robust than the method given in \citep{coulom2006efficient} as at any given time the interpolation is taken between the soft maximum and the averaged value, rather than between the averaged value and the hard maximum which is volatile to outliers in the returns. \par

Another noteworthy point is that in our method, as well as in \citep{coulom2006efficient} and \citep{khandelwal2016analysis}, the best performance comes from an update rule which invariably underestimates the max function, which leads us to hypothesize that the main trade-off between weighting the best child more or less heavily may not be so much a question of overestimation or underestimation, but rather one of robustness. The experiments in the next section will demonstrate that Softmax MCTS outperforms the method in \citep{coulom2006efficient} over a number of different parameters.
\vspace{-0.3cm}
\section{Experiments}
\vspace{-0.1cm}
In this section, we use 9x9 and 19x19 Go as a testbed to run Monotone MCTS and Softmax MCTS against several methods in the literature. We first establish a baseline by running these methods against standard MCTS. The table below records the win-rates (\%).

\begin{table} [H]
\centering
\caption{Win-rates (\%) of the various methods versus standard MCTS on 9x9 and 19x19 Go.}
\begin{tabular}{|c|c|c|} 
\hline
&  9 x 9 Go & 19 x 19 Go \\ 
\hline
Coulom (2, 16) & 44.9 & 53.3 \\
\hline
Coulom (4, 32) & 45.8 & 51.2 \\ 
\hline
Coulom (8, 64) & 48.3 & 50.2 \\
\hline
$GAY$ & 50.2 & 56.1 \\
\hline
$GBY$ & 51.8 & 54.5 \\
\hline 
ERWA, $\alpha = 0.00001$ & 51.2 & 53.9 \\
\hline 
ERWA, $\alpha = 0.0001$ & 50.3 & 56.9 \\
\hline
ERWA, $\alpha = 0.001$ & 52.9 & 55.1 \\
\hline
\end{tabular}
\end{table}

Coulom$(x, y)$ refers to the method in \citep{coulom2006efficient}, where $x$ is proportional to the mean-weight parameter in \eqref{coulom} and $y$ controls when it begins increasing. $GAY$ and $GBY$ are the best feedback adjustment policies in \citep{xie2009backpropagation} (see Figure 1), and we also test ERWA at various parameters of $\alpha$. \par 

All experiments are run with 5000 moves per simulation for 9x9 Go, and 1600 moves per simulation for 19x19 Go. For this and all subsequent tests, the win-rates are computed based on 1000 games.  \par 

We follow the architecture in \citep{silver2017mastering} for the neural nets. This consists of an input convolutional layer, followed by several layers of residual blocks with batch normalization (4 layers for 9x9, 10 layers for 19x19), followed by two "heads", one which outputs the policy vector and the other the value of the position. Both heads start with a convolutional layer, and is followed by a fully-connected layer before the output. The input layer has 10 (18 for 19x19) channels encoding the current position and the previous 4 (8 for 19x19) positions. In each residual block, we used 64 filters for 19x19 in the convolutional layers and 32 filters for 9x9. \par 

The neural net for 9x9 was trained tabula rasa using reinforcement learning \citep{silver2017mastering} over 600,000 training steps, with each step processing a minibatch of 16 inputs. We trained the 19x19 neural net by supervised learning over the GoGod database of approximately 15 million datapoints from 80,000 games. \par 

In all tests, we use PUCT \eqref{puct} with the exploration constant set to 0.5, and weight the exploration term by the distribution given by the policy vector output of the neural networks \citep{silver2016mastering, silver2017mastering}.

\subsection{Monotone MCTS}

We run Spearmint with $m = 6$ parameters to optimize the win-rate of Monotone MCTS against standard MCTS. For every set of parameters, 400 games were run to determine the win-rate during the optimization phase (we run 1000 games for testing). The tables below record the results for 9x9 Go of Monotone MCTS versus standard MCTS, ERWA and the feedback adjustment policies in \citep{xie2009backpropagation} for the two best sets of parameters found.
\vspace{-10pt}
\begin{table} [H]
\centering
\caption{Win-rates (\%) for Monotone MCTS with the 1st set of parameters (-10.0, -10.0, -4.0, -4.0, -4.0, -10.0) versus various methods (left column) in 9x9 Go.}
\begin{tabular}{|c|c|} 
\hline
Standard MCTS & \textbf{53.1} \\
\hline
ERWA, $\alpha = 0.00001$ & 50.3 \\ 
\hline
ERWA, $\alpha = 0.0001$ & 52.8 \\ 
\hline
ERWA, $\alpha = 0.001$ & 52.8 \\ 
\hline
$GAY$ & 51.5 \\ 
\hline
$GBY$ & 51.0 \\
\hline 
\end{tabular}
\end{table}
\vspace{-15pt}
\begin{table} [H]
\centering
\caption{Win-rates (\%) for Monotone MCTS with the 2nd set of parameters (-4.0, -4.0, -4.0, -10.0, -4.0, -4.0) versus various methods (left column) in 9x9 Go.}
\begin{tabular}{|c|c|} 
\hline
Standard MCTS & \textbf{54.5} \\
\hline
ERWA, $\alpha = 0.00001$  & 52.3 \\ 
\hline
ERWA, $\alpha = 0.0001$  & 50.5 \\ 
\hline
ERWA, $\alpha = 0.001$  & 53.1 \\ 
\hline
$GAY$  & 52.3 \\
\hline 
$GBY$  & 51.1 \\
\hline 
\end{tabular}
\end{table}
\vspace{-10pt}
We also test Monotone MCTS against standard MCTS in 19x19 Go and find that the first set of parameters achieves a win-rate of 56.0\%, whereas the second set of parameters achieves a win-rate of 54.3\%. \par 

The figure below shows a graph of the weight profiles.

\begin{figure} [H]
\centering
\includegraphics[scale=0.45]{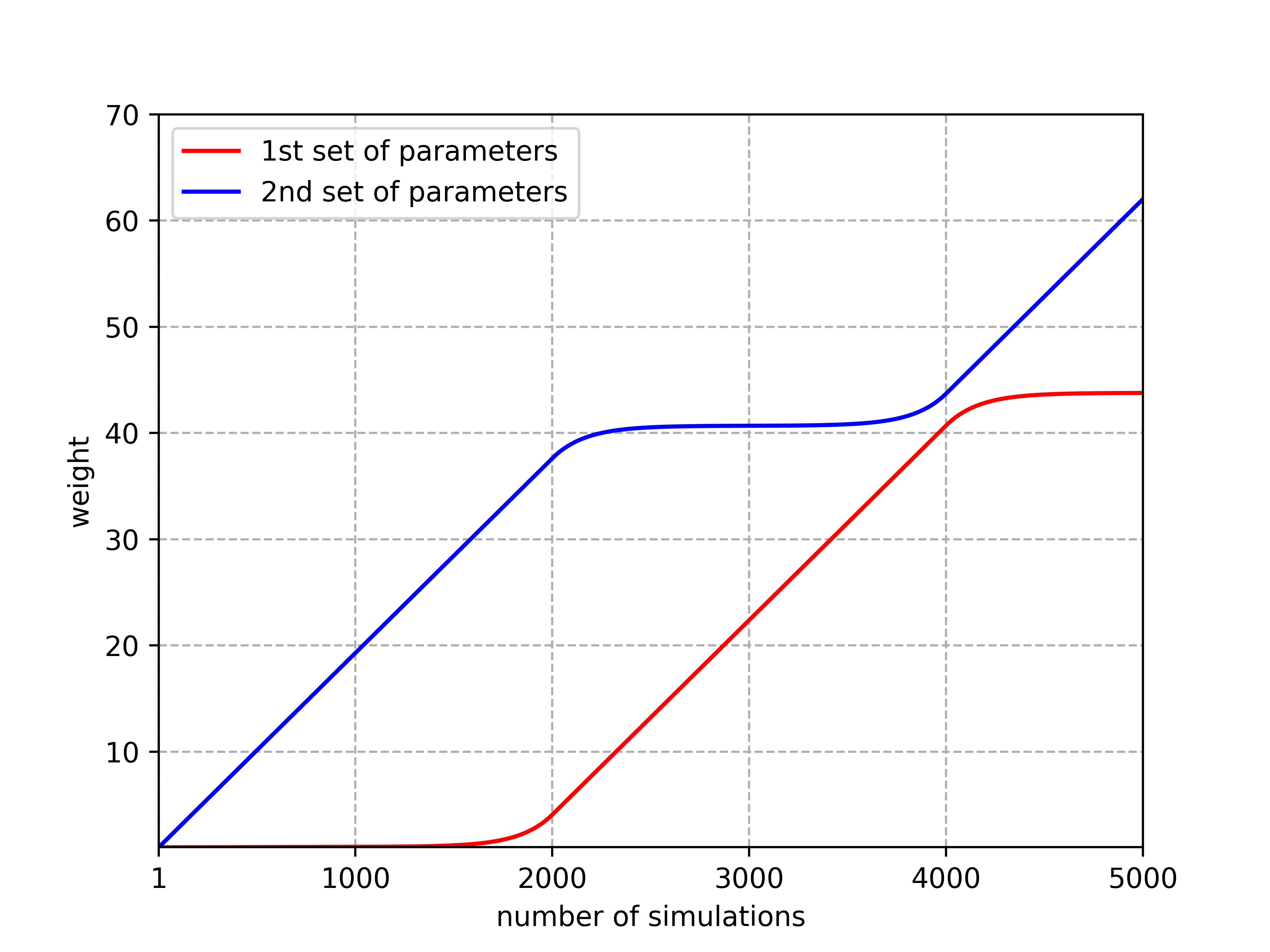}
\caption{Weight profiles of Monotone MCTS for both sets of parameters.}
\end{figure}

\vspace{-0.8cm}
\subsection{Softmax MCTS}
As in the previous section, Spearmint was run with $m=6$ parameters to optimize the win-rate of Softmax MCTS versus standard MCTS. We present the results of the two best sets of parameters in the tables below.

\begin{table} [H]
\centering
\caption{Win-rates (\%) for Softmax MCTS with the 1st set of parameters (-4.0, -10.0, -4.0, -4.0, -10.0, -4.0) versus various methods (left column) in 9x9 Go.}
\begin{tabular}{|c|c|}
\hline
Standard MCTS & 56.3 \\
\hline
Coulom(2,16) & \textbf{59.3} \\ 
\hline
Coulom(4, 32) & 55.5  \\
\hline
Coulom(8, 64) & 55.1 \\
\hline
\end{tabular}
\end{table}
\vspace{-10pt}
\begin{table} [H]
\centering
\caption{Win-rates (\%) for Softmax MCTS with the 2nd set of parameters (-4.0, -10.0, -7.9, -10.0, -10.0, -7.8) versus various methods (left column) in 9x9 Go.}
\begin{tabular}{|c|c|} 
\hline
Standard MCTS & 57.8 \\
\hline
Coulom(2,16) & 57.6 \\ 
\hline
Coulom(4,32) & \textbf{59.5} \\
\hline
Coulom(8,64)  & 51.9 \\
\hline
\end{tabular}
\end{table}
\vspace{-10pt}
We also test Softmax MCTS against standard MCTS in 19x19 Go and find that the first set of parameters achieves a win-rate of 53.2\%, whereas the second set of parameters achieves a win-rate of 55.9\%. \par

\begin{figure} [H]
\centering	
\includegraphics[scale=0.45]{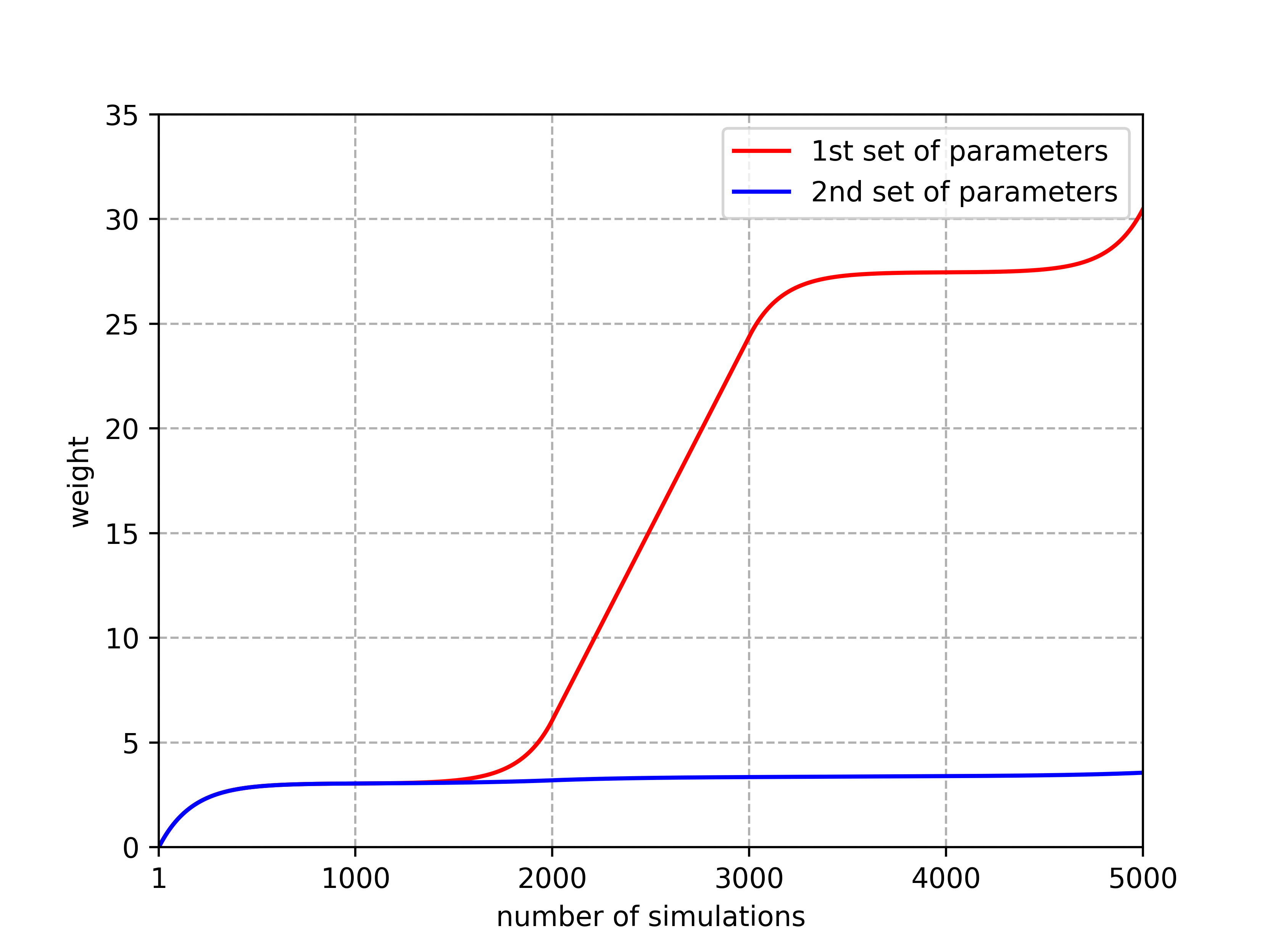}
\caption{Weight profiles of Softmax MCTS for both sets of parameters.}
\end{figure}

\vspace{-0.8cm}
\section{Conclusion and Future Work}
\vspace{-0.1cm}
In this paper, we present a unifying framework for backpropagation strategies in MCTS for min-max trees. Our proposed method allows one to perform optimization in two orthogonal directions. The first algorithm Softmax MCTS allows one to find the optimal schedule that weights the best child more gradually as the tree grows, and the second method Monotone MCTS generalizes previous work in adapting the update rule to get the most accurate estimate of a node's value in a non-stationary setting. \par 
Doing so requires optimization over the space of monotone functions, a high-dimensional problem we overcome efficiently by using parallelized Bayesian optimization over a Gaussian process prior. Once the parameters that define the optimal monotone function are found, they can be incorporated into MCTS with negligible overhead. Our experiments show that this new approach is superior to previous methods in the literature. \par 
To conclude, we would like to note also that it is possible, indeed advisable, to perform the optimization in conjunction with the exploration constant in the selection phase, but we have decided in this paper to focus solely on the backpropagation phase and to elucidate the effects of different monotone weight profiles on the win-rate. Combining these optimal backup strategies with other phases of MCTS will be the topic of future work.


\end{document}